\documentclass[10pt,twocolumn,letterpaper]{article}

\usepackage[pagenumbers]{cvpr} % To force page numbers, e.g. for an arXiv version

\usepackage{mathtools}
\usepackage{threeparttable}
\usepackage{xcolor}
\usepackage[ruled, lined, linesnumbered, commentsnumbered, longend]{algorithm2e}
\usepackage{listings}
\usepackage{graphicx}
\usepackage{amssymb}
\usepackage{booktabs}
\usepackage{array}
\usepackage{amsmath}
\usepackage{amsfonts}
\usepackage{graphicx}
\usepackage{caption}
\usepackage[export]{adjustbox}
\usepackage{makecell}
\usepackage{multirow}
\usepackage{rotating}
\usepackage{environ}
\usepackage{enumerate}
\usepackage{amsthm}
\usepackage{longtable}

\newcommand{\ours}{MDT-dist}
\newtheorem{theorem}{Theorem}

\renewcommand{\thefootnote}{\fnsymbol{footnote}}

\definecolor{cvprblue}{rgb}{0.21,0.49,0.74}
\usepackage[pagebackref,breaklinks,colorlinks,allcolors=cvprblue]{hyperref}

\def\paperID{*****} % *** Enter the Paper ID here
\def\confName{CVPR}
\def\confYear{2025}

\title{Few-step Flow for 3D Generation via Marginal-Data Transport Distillation}

\author{
	Zanwei Zhou$^{1,*}$\quad Taoran Yi$^{2,*}$\quad Jiemin Fang$^{3,\dagger}$\quad Chen Yang$^{3}$\quad Lingxi Xie$^{3}$\\ Xinggang Wang$^{2}$\quad Wei Shen$^{1,\dagger}$\quad Qi Tian$^{3}$\vspace{8pt}\\
	$^{1}$Shanghai Jiao Tong University\quad $^{2}$Huazhong University of Science and Technology\quad $^{3}$Huawei Inc.\\
    \texttt{\small\{sjtu19zzw, wei.shen\}@sjtu.edu.cn}\;\;
    \texttt{\small\{taoranyi, xgwang\}@hust.edu.cn}\\  
    \texttt{\small \{jaminfong, chenyang.res, 198808xc\}@gmail.com} \;\;
    \texttt{\small tian.qi1@huawei.com} \\
    \url{https://github.com/Zanue/MDT-dist}
}

\begin{document}
\maketitle

\let\thefootnote\relax\footnote{$^*$ Equal contribution. Work done during internship at Huawei.}
\let\thefootnote\relax\footnote{$^\dagger$ Corresponding author.}

\begin{abstract}
Flow-based 3D generation models typically require dozens of sampling steps during inference. 
Though few-step distillation methods, particularly Consistency Models (CMs), have achieved substantial advancements in accelerating 2D diffusion models, they remain under-explored for more complex 3D generation tasks. 
In this study, we propose a novel framework, \textbf{MDT-dist}, for few-step 3D flow distillation. 
Our approach is built upon a primary objective: \underline{\textbf{dist}}illing the pretrained model to learn the \underline{\textbf{M}}arginal-\underline{\textbf{D}}ata \underline{\textbf{T}}ransport. 
Directly learning this objective needs to integrate the velocity fields, while this integral is intractable to be implemented. Therefore, we propose two optimizable objectives, Velocity Matching (VM) and Velocity Distillation (VD), to equivalently convert the optimization target from the transport level to the velocity and the distribution level respectively. 
Velocity Matching (VM) learns to stably match the velocity fields between the student and the teacher, but inevitably provides biased gradient estimates. 
Velocity Distillation (VD) further enhances the optimization process by leveraging the learned velocity fields to perform probability density distillation.
When evaluated on the pioneer 3D generation framework TRELLIS, our method reduces sampling steps of each flow transformer from 25 to 1–2, achieving 0.68s (1 step $\times$ 2) and 0.94s (2 steps $\times$ 2) latency with 9.0$\times$ and 6.5$\times$ speedup on A800, while preserving high visual and geometric fidelity. 
Extensive experiments demonstrate that our method significantly outperforms existing CM distillation methods, and enables TRELLIS to achieve superior performance in few-step 3D generation. 
\end{abstract}

\section{Introduction}
\label{sec:intro}

Flow-based 3D generation models~\cite{trellis,zhao2025hunyuan3d2,yang2024hunyuan3d1,li2025step1x,wu2024direct3d,wu2025direct3ds2,chen2025dora,ye2025hi3dgen,zhang2024clay} have exhibited remarkable abilities in synthesizing intricate 3D representations from image prompts. However, during inference they typically require dozens of iterative sampling steps, posing significant computational barriers to practical applications such as large-scale 3D content generation for embodied intelligence simulation~\cite{wang2025embodiedgengenerative3dworld} and real-time interactive editing workflows in graphics systems. 
Although few-step diffusion distillation methods, particularly consistency models (CMs)~\cite{song2023consistency, wang2024phased, lu2024simplifying}, have achieved substantial advancements in accelerating 2D diffusion models, their extension to the 3D generation area remains underexplored. The only related work recently is FlashVDM~\cite{lai2025flashvdm}, which adopts a few-step distillation framework mainly derived from the previous Phased Consistency Models (PCM) method~\cite{wang2024phased}.

3D generation presents inherently greater challenges than its 2D counterpart. Unlike 2D images sampled from a continuous color space, 3D representations, \eg, meshes and 3D Gaussians~\cite{kerbl3Dgaussians}, are discrete and sparsely structured in 3D space. 3D models also contain richer geometric and textural details at a finer granularity. Moreover, in latent-space generative frameworks such as Latent Diffusion models (LDM)~\cite{rombach2022high}, the dimension of the 3D latent space is typically higher than that of the 2D latent space. These fundamental differences indicate that 3D generation faces more difficulties and challenges than 2D generation, and thus has stricter requirements on few-step acceleration techniques.

To address these challenges, we propose a novel framework, MDT-dist, for few-step 3D flow distillation.
Our method is built upon a primary objective: distilling the pretrained 3D flow model to learn the marginal-data transport. CMs have a similar optimization target, but are limited by the consistency constraint which enforces consistency between adjacent time steps to indirectly learn the target. We instead propose two novel loss functions, Velocity Matching (VM) and Velocity Distillation (VD), to pursue the primary objective in a more direct way. Directly learning the primary objective needs to integrate the velocity fields, but this integral is intractable to be implemented. Therefore, VM and VD equivalently convert it to tractable objectives respectively. 
In VM, the optimization of the primary objective is converted into optimizing its time derivative, with the error in the primary objective bounded by the error in the VM loss. Specifically, VM learns to stably match the velocity fields between the student and the teacher. 
However, it inevitably contains a term involving the derivative of the network output. This term cannot be efficiently back-propagated and is therefore detached, leading to biased gradient estimates.
% However, VM contains a term involving the derivative of the network output, which is difficult to be efficiently back-propagated in modern deep learning frameworks (\eg, PyTorch~\cite{paszke2019pytorch}) whether using continuous or discrete approximations. In practice the gradient of the derivative item is stopped, which makes the optimization process efficient but inevitably providing a biased gradient direction.
VD further enhances the learning of the marginal-data transport by matching the marginal distributions between the student and the teacher. It leverages the velocity fields learned by the student and the teacher as measure to perform probability density distillation.
We evaluate our methods on a state-of-the-art 3D generation framework TRELLIS~\cite{trellis}. Our approach reduces the inference steps of each flow transformer from 25 to just 1-2 and the latency from 6.1s to 0.68s and 0.94s on A800, while preserving high visual and geometric fidelity. Extensive experiments demonstrate that our method significantly outperforms existing CM distillation methods, and makes the distilled TRELLIS model surpass FlashVDM~\cite{lai2025flashvdm}, enabling fast 3D content generation beneficial for various downstream tasks.

Our contributions are as follows:
\begin{itemize}
    \item We develop a novel few-step flow distillation framework MDT-dist for better 3D generation acceleration.
    \item We propose two novel optimization objectives, Velocity Matching (VM) and Velocity Distillation (VD), to jointly enable effective few-step distillation.
    \item We distilled TRELLIS to achieve a sweet balance between generation speed and quality. Our method reduces sampling steps of each flow transformer from 25 to 1–2, achieving 0.68s (1 step $\times$ 2) and 0.94s (2 steps $\times$ 2) latency with 9.0$\times$ and 6.5$\times$ speedup on A800.
\end{itemize}

\section{Related Work}
\paragraph{3D Generation Models.}
Early 3D generation methods~\cite{jain2022zero,michel2022text2mesh,lei2022tango,lin2023magic3d,chen2023fantasia3d,shi2023mvdream,armandpour2023re,liu2023syncdreamer,seo2023ditto,seo2023let,wu2023hd,tsalicoglou2023textmesh,raj2023dreambooth3d,huang2023tech,ouyang2023chasing,metzer2023latent,cao2023dreamavatar,huang2023dreamwaltz,zhang2023avatarverse,jiang2023avatarcraft,zhang2023dreamface,song2023roomdreamer,liu2023zero1to3,long2023wonder3d} are mainly based on 2D diffusion models~\cite{rombach2022high}, generating 3D assets by iteratively prompting 2D diffusion models to optimize 3D representations~\cite{mildenhall2020nerf,kerbl3Dgaussians}. DreamFusion~\cite{poole2022dreamfusion} and Score Jacobian Chaining (SJC)~\cite{song2020score} first introduce Score Distillation Sampling (SDS) to generate 3D assets using pretrained 2D diffusion models. ProlificDreamer~\cite{wang2023prolificdreamer} and other methods~\cite{EnVision2023luciddreamer,sun2023dreamcraft3d,zhao2023efficientdreamer} further improve SDS to achieve better generation results.
Some methods~\cite{ma2023geodream,qiu2024richdreamer,chen2023gsgen,yi2023gaussiandreamer,tang2023dreamgaussian,yi2024gaussiandreamerpro,li2023sweetdreamer} incorporate shape priors to significantly reduce generation time.
% while achieving good quality, but still require optimization and may result in geometrically unsuccessful 3D assets. 
Methods~\cite{jun2023shap,nichol2022point,gupta3dgen,gao2022get3d,luo2023scalable,shen2024gamba,li2023instant3d,hong2023lrm,tang2024lgm,zou2023triplane,xu2024grm,gslrm2024} like LRM~\cite{hong2023lrm} and LGM~\cite{tang2024lgm} build native 3D generative models by pretraining on large-scale 3D data, enabling feed-forward generation of 3D assets without optimization. Some native 3D generation methods~\cite{trellis,zhao2025hunyuan3d2,yang2024hunyuan3d1,li2025step1x,wu2024direct3d,wu2025direct3ds2,3dtopiaxl,chen2025sar3d,chen2025dora,ye2025hi3dgen,chen2025mar,zhang2024clay,li2024craftsman} further introduce flow matching into the 3D generation field.
However, generating high-fidelity 3D assets with 3D diffusion models requires a relatively large number of sampling steps during inference, which both increases users' queuing time and raises computational costs. FlashVDM~\cite{lai2025flashvdm} shortens the generation time through efficient decoder design and few-step distillation, but it can only generate the shape without appearance. We apply our method to TRELLIS~\cite{trellis} to reduce the number of sampling steps for the two stage flow transformers, enabling the fast generation on both shape and appearance.

\paragraph{2D Generative Models.} 

2D generative modeling has progressed from variational autoencoders (VAEs)~\cite{kingma2013auto,rezende2014stochastic} to generative adversarial networks (GANs)~\cite{goodfellow2020generative}. VAEs map data distributions to latent Gaussian spaces via Evidence Lower Bound (ELBO) optimization but often yield blurry outputs. GANs employ adversarial training to produce high-fidelity images but suffer from instability and mode collapse. 
Diffusion models have since become a dominant paradigm: denoising diffusion probabilistic models (DDPM)~\cite{ho2020denoising} formalize iterative denoising as a Markov process; score-based generative modeling~\cite{song2020score} unifies diffusion under the framework of stochastic differential equations (SDEs); flow matching~\cite{lipman2022flow, liu2022flow, tong2023conditional} learns velocity fields for direct ODE-based generation and often requires fewer steps than diffusion.
An important topic on diffusion models is acceleration by reducing the number of sample steps. Denoising diffusion implicit models (DDIM)~\cite{song2020denoising} accelerate generation through non-Markovian updates. DPM-Solver~\cite{lu2022dpm} further reduces sampling steps by employing higher-order ODE solvers. Consistency models (CMs)~\cite{song2023consistency} enable single-step sampling by enforcing consistency across sampling trajectories. Variants such as phased consistency models (PCM)~\cite{wang2024phased} and Trigflow(sCM)~\cite{lu2024simplifying} improve stability through phased training strategies and continuous-time formulations, respectively. In contrast to consistency models, our method is derived without the consistency constraint and serves as a novel few-step distillation framework.

\section{Background}

\subsection{Diffusion Models}
% easy to understand!!!
Diffusion models~\cite{ho2020denoising} learn to generate data by iteratively denoising samples from Gaussian noise distribution. The framework defines a fixed forward diffusion process and a learned reverse denoising process. Given a data sample $\boldsymbol x_0 \sim q_{\text{data}}$, the forward diffusion process gradually adds the Gaussian noise and produces a series of noised samples $\{\boldsymbol x_t \}_{t=1}^T$, conditioned on the time step $t$. This induces a sequence of marginal distributions $q_t(\boldsymbol x_t)$, \ie, the noised distribution at time $t$. The reverse process is parameterized by a neural network with learnable parameters $\theta$, which learns to generate the denoising direction. During the reverse process, the generated marginal distribution $p_{\theta}^t(\boldsymbol x_t)$ is expected to be matched with $q_t(\boldsymbol x_t)$.

\paragraph{Flow Matching.}
Flow matching (FM)~\cite{lipman2022flow, liu2022flow, tong2023conditional} learns to map noise to data distribution by estimating a Probability Flow Ordinary Differential Equation (PF-ODE) process. It defines a continuous-time dynamical system with a learnable velocity field $\boldsymbol v_\theta(\cdot, t), t \in [0, 1]$, which can be used to construct a time-dependent diffeomorphic map $\boldsymbol \phi_t$,
\begin{equation}\begin{aligned}
\frac{\mathrm d}{\mathrm d t}\boldsymbol \phi_t(\boldsymbol x)&=\boldsymbol v_{\theta}(\boldsymbol \phi_t(\boldsymbol x), t),\\ 
\boldsymbol \phi_0(\boldsymbol x)&= \boldsymbol x,
\end{aligned}\end{equation}
which subsequently defines a push-forward $\boldsymbol \phi_{*}$ transforming a density over time
\begin{equation}
p_{\theta}^t(\boldsymbol x)=[\boldsymbol \phi_t]_{*}p_0(\boldsymbol x)=p_0(\boldsymbol \phi^{-1}(\boldsymbol x))\left|\det\nabla_{\boldsymbol x}\boldsymbol \phi_t^{-1}(\boldsymbol x)\right|.
\end{equation}
In this way, the velocity field $\boldsymbol v_\theta(\cdot, t)$ is said to generate a probability path $ p_{\theta}^t$.
The velocity field is optimized by minimizing the conditional flow matching loss
\begin{equation}
\mathcal{L}_{\text{CFM}}(\theta) = \mathbb{E}_{t, \boldsymbol x_0, \boldsymbol z} \left[ \left\| \boldsymbol v_\theta\left((1-t)\boldsymbol x_0 + t\boldsymbol z, t\right) - (\boldsymbol z - \boldsymbol x_0)\right\|^2 \right].
\end{equation}
At inference, samples are generated from a Gaussian noise sample $\boldsymbol{z} \sim \mathcal{N}(\boldsymbol 0, \boldsymbol{I})$ by solving the ODE backward in time:  
\begin{equation}
\boldsymbol x_0 = \boldsymbol z - \int_{0}^1 \boldsymbol v_{\theta}(\boldsymbol x_{\tau}, \tau) \mathrm d \tau.
\end{equation}
Notably, the integral of $\boldsymbol v_{\theta}(\boldsymbol x_{t}, t)$ on a time interval $[t_1, t_2]$ indeed describes the transport from $p^t_{\theta}(\boldsymbol x_{t_2})$ to $p^t_{\theta}(\boldsymbol x_{t_1})$:
\begin{equation}
\boldsymbol x_{t_1} = \boldsymbol x_{t_2} - \int_{t_1}^{t_2} \boldsymbol v_{\theta}(\boldsymbol x_{\tau}, \tau) \mathrm d \tau.\label{eq:flow_transition}
\end{equation}
In practice, the ODE is discretely approximated using a numerical ODE solver such as Euler's method,
\begin{equation}
\boldsymbol x_{t_1} = \boldsymbol x_{t_2} -  \sum_{k = 1}^{N} \boldsymbol v_{\theta}(\boldsymbol x_{\tau_k}, \tau_k) \Delta \tau,
\end{equation}
where $t_1 = \tau_{1} < \tau_2 < \dots < \tau_N = t_2$, $\boldsymbol x_{\tau_{k-1}} = \boldsymbol x_{t_{\tau_{k}}} - \boldsymbol v_{\theta}(\boldsymbol x_{t_{\tau_{k}}}, \tau_{k}) \Delta \tau$, and $\Delta \tau = \tau_{k} - \tau_{k-1}$.

\subsection{3D Generation Models}
Our method builds upon TRELLIS~\cite{trellis}, a recent high-quality 3D asset generation framework. In TRELLIS, a 3D asset is implicitly represented by a structured latent variable (SLAT) $\boldsymbol{S}$, which is composed of sparse voxels and features:
\begin{equation}
	\label{eq: slat}
  \boldsymbol{S} = \{(f_i,\boldsymbol{p}_i)\}_{i=1}^{L},\quad f_i\in\mathbb{R}^C, \ \boldsymbol{p}_i\in \{0, 1,\ldots, N-1\} ^3, 
\end{equation}
where $\boldsymbol{p}_i$ is the coordinate of the $i$-th voxel, and $f_i$ is the corresponding feature. $C$ denotes the feature dimension, $N$ denotes the voxel resolution, and $L$ denotes the number of active voxels which is much smaller than $N^3$.
SLAT can be decoded into different 3D representations such as 3D Gaussians~\cite{kerbl3Dgaussians}, meshes and NeRFs~\cite{mildenhall2020nerf}. For generation, two flow transformers are trained to generate coordinates and features of SLAT separately. During inference, the two models both take 25 sampling steps by default.

\begin{figure*}[thbp]
\centering
\includegraphics[width=1\textwidth]{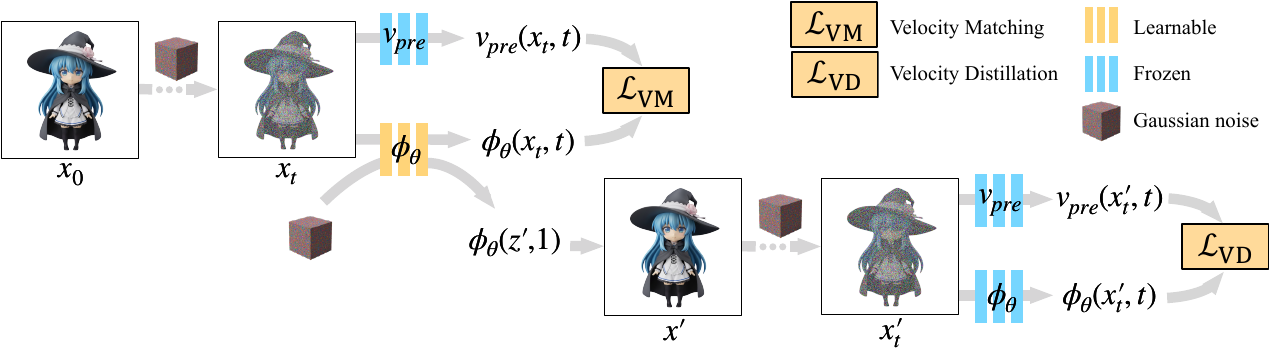}
\vspace{-20pt}
\caption{The primary objective of our framework is to learn the transport from the marginal distribution to the data distribution. Based on it we propose two optimization objectives: velocity matching and velocity distillation. Velocity matching directly supervises $\boldsymbol{\phi}_{\theta}$ via matching the velocity fields between the student and teacher (Eq.~\ref{eq:loss_func1}), while velocity distillation indirectly supervises $\boldsymbol{\phi}_{\theta}$ via matching the marginal distributions between the student and teacher (Eq.~\ref{eq:loss_func2}).}
\vspace{-5pt}
\label{fig: framework}
\end{figure*}

\section{Marginal-Data Transport Distillation}
Our target is to distill a pretrained 3D flow model into a few-step generator under the guidance of a pretrained teacher model. Unlike consistency models limited by the consistency constraint, we instead seek a more direct objective.

A straightforward objective is to directly learn the one-step mapping from the noise distribution to the data distribution. However, it has been validated by VAEs~\cite{kingma2013auto} and GANs~\cite{goodfellow2020generative} that directly associating the noise distribution to the data distribution faces unstable optimization issues such as posterior collapse~\cite{lucas2019don} and mode collapse~\cite{che2017mode, srivastava2017veegan}. Inspired by the insight of iterative denoising in diffusion models~\cite{ho2020denoising}, we propose to enhance the objective by learning the mapping from any marginal distribution $q_t(\boldsymbol x_t)$ to the data distribution $q_{\text{data}}$.

Specifically, in Sec.~\ref{sec:objective} our primary objective is set to be learning the marginal-data transport which directly transports the marginal distribution $q_t(\boldsymbol x_t)$ to the data distribution $q_{\text{data}}$. Compared with the objective from VAEs and GANs, our primary objective ensures a progressive learning of the noise-data transport, making the optimization process more stable. Unfortunately, simulating this mapping function needs to integrate the velocity fields which is intractable and cannot be directly optimized.
Therefore, starting from the primary objective, we develop two optimizable objectives: velocity matching (Sec.~\ref{sec:vm}) and velocity distillation (Sec.~\ref{sec:vd}).
These two objectives equivalently convert the optimization target from the transport level to the velocity level and the distribution level respectively. Velocity matching learns to stably match the velocity fields between the student and the teacher, and velocity distillation further enhances the optimization process by leveraging the learned velocity fields to perform probability density distillation. We combine these two objectives into a final loss function in Sec.~\ref{sec:final_loss}. We also clarify the differences between our approach and previous methods in Sec.~\ref{sec:relation_prev}.
Our framework is presented in Fig.~\ref{fig: framework}.

\subsection{Primary Objective}\label{sec:objective}
Given a noise sample $\boldsymbol z \sim \mathcal{N}(\boldsymbol 0, \boldsymbol I)$, our target is to learn a neural network $\boldsymbol \phi_{\theta}$ to generate $\boldsymbol{x}_0 \sim q_{\text{data}}$ within few-step forwards. Taking one-step generation as an example, $\boldsymbol \phi_{\theta}$ is expected to satisfy
\begin{equation}
    \min_{\theta}\mathbb{E}_{\boldsymbol x_0, \boldsymbol z}\left[D\left(\boldsymbol \phi_{\theta}(\boldsymbol{z}), \boldsymbol{z} - \boldsymbol{x}_0\right)\right],\label{eq:one_step_tgt}
\end{equation}
where $D(\cdot, \cdot)$ represents a distance metric such as MSE or Kullback-Leibler Divergence, and $\theta$ represents learnable model parameters. This formulation aims to transport the noise distribution to the data distribution directly. However, from the experiences of VAEs~\cite{kingma2013auto} and GANs~\cite{goodfellow2020generative}, it is well known that directly learning data distribution from noise distribution is a challenging task and easily leads to unstable optimization issues such as posterior or mode collapse. Therefore, we instead force $\boldsymbol \phi_{\theta}$ to learn the transport from the marginal distribution $q_t(\boldsymbol{x}_t)$ to the data distribution $q_{\text{data}}$ for any diffusion time step $t$:
\begin{equation}
    \min_{\theta}\mathbb{E}_{t, \boldsymbol x_0, \boldsymbol z}\left[D\left(t\boldsymbol \phi_{\theta}(\boldsymbol{x}_t, t), \boldsymbol{x}_t - \boldsymbol{x}_0 \right)\right],\label{eq:marginal_tgt}
\end{equation}
where $t$ is used for normalization, $\boldsymbol{x}_t = (1-t)\boldsymbol{x}_0 + t \boldsymbol{z}$. With the guidance of a well-pretrained 3D flow model, we can approximate $q_t(\boldsymbol{x}_t)$ using its learned reverse marginal distribution, and approximate the data distribution $q_{\text{data}}$ using the generated teacher distribution. 
Recall Eq.~\ref{eq:flow_transition}, the transport from $p_{\theta}^t(\boldsymbol{x}_t)$ to $p_0$ is $\int_{0}^{t} \boldsymbol v_{\theta}(\boldsymbol x_{\tau}, \tau) \mathrm d \tau$, thus our \textit{primary objective} is formulated as
\begin{multline}\mathcal{L}_{\text{primary}}(\theta) \coloneqq \\
\min_{\theta}\mathbb{E}_{t, \boldsymbol x_0, \boldsymbol z}\left[D\left(t\boldsymbol  \phi_{\theta}(\boldsymbol{x}_t, t), \int_{0}^{t} \boldsymbol{v}_{\textrm{pretrain}}(\boldsymbol{x}_{\tau}, \tau) \mathrm{d}\tau\right)\right],\label{eq:loss_tgt}
\end{multline}
where $\boldsymbol{v}_{\textrm{pretrain}}$ denotes the velocity fields predicted by the pretrained teacher model.

\subsection{Velocity Matching}\label{sec:vm}
% explain vm

Note that the primary objective in Eq.~\ref{eq:loss_tgt} cannot be directly optimized, since the integral $\int_{0}^{t} \boldsymbol{v}_{\textrm{pretrain}}(\boldsymbol{x}_{\tau}, \tau) \mathrm{d}\tau$ is intractable. We differentiate the objective function with respect to $ t $ and turn it to be
\begin{gather}
\min_{\theta}\mathbb{E}_{t, \boldsymbol{x}_0, \boldsymbol z}\left[ D\left( \boldsymbol u_{\theta}(\boldsymbol{x}_t, t), \boldsymbol{v}_{\textrm{pretrain}}(\boldsymbol{x}_{t}, t) \right)\right],\label{eq:vm} \\
\boldsymbol u_{\theta}(\boldsymbol{x}_t, t) = \boldsymbol \phi_{\theta}(\boldsymbol{x}_t, t) + t\frac{\mathrm d \boldsymbol \phi_{\theta}(\boldsymbol{x}_t, t)}{\mathrm d t}.\label{eq:u_theta}
\end{gather}
Here $\boldsymbol u_{\theta}(\boldsymbol{x}_t, t)$ is the derivative of $t\boldsymbol  \phi_{\theta}(\boldsymbol{x}_t, t)$ with respect to $t$, which actually represents the velocity fields.
Intuitively, Eq.~\ref{eq:loss_tgt} supervises $\boldsymbol \phi_{\theta}$ by the transport from the marginal distribution to the data distribution, and Eq.~\ref{eq:vm} converts it to be a supervision on its derivative, \ie, the velocity supervision, to learn student velocity fields matched with the teacher velocity fields. 

Now Eq.~\ref{eq:vm} is tractable and can be directly used for supervision. Given a data sample $\boldsymbol{x}_0 \sim q_{\text{data}}$, we replace $D(\cdot, \cdot)$ with the MSE metric and define our velocity matching loss as 
\begin{multline}\mathcal{L}_{\text{VM}}(\theta) \coloneqq \\
\min_{\theta}\mathbb{E}_{t, \boldsymbol{x}_0, \boldsymbol z}\left[\left\|\boldsymbol \phi_{\theta}(\boldsymbol{x}_t, t) + t\frac{\mathrm d \boldsymbol \phi_{\theta}(\boldsymbol{x}_t, t)}{\mathrm d t} - \boldsymbol{v}_{\textrm{pretrain}}(\boldsymbol{x}_{t}, t) \right\|^2\right].\label{eq:loss_func1}
\end{multline}
We provide a detailed proof in Appendix~\ref{sec:proof_vm} to demonstrate that, the error in the primary objective $\mathcal{L}_{\text{primary}}(\theta)$ is bounded by the error in the velocity matching loss $\mathcal{L}_{\text{VM}}(\theta)$. Therefore, we can effectively learn the primary objective through optimizing the velocity matching loss. 

In practice, we accelerate the convergence by discretely approximating the derivative item
\begin{equation}
\frac{\mathrm d \boldsymbol \phi_{\theta}(\boldsymbol{x}_t, t)}{\mathrm d t} \approx \frac{1}{\Delta t} \left( \boldsymbol \phi_{\theta}(\boldsymbol{x}_{t}, t) - \boldsymbol \phi_{\theta}(\boldsymbol{x}_{t-\Delta t}, t-\Delta t) \right),
\end{equation}
where $\boldsymbol{x}_{t-\Delta t} = \boldsymbol{x}_{t} -  \boldsymbol{v}_{\textrm{pretrain}}(\boldsymbol{x}_{t}, t)\Delta t $, $\Delta t$ is a small constant value and we set it to be $1e-2$. The gradient of the term $\boldsymbol \phi_{\theta}(\boldsymbol{x}_{t-\Delta t}, t-\Delta t)$ is expected to be detached, since solving it requires computing two-order derivative which is computationally expensive.
When training, we stop the gradient of the derivative term $\frac{\mathrm d \boldsymbol \phi_{\theta}(\boldsymbol{x}_t, t)}{\mathrm d t}$ rather than only $\boldsymbol \phi_{\theta}(\boldsymbol{x}_{t-\Delta t}, t-\Delta t)$. This operation makes $\boldsymbol \phi_{\theta}(\boldsymbol{x}_t, t)$ to learn more consistent with $\boldsymbol{v}_{\textrm{pretrain}}(\boldsymbol{x}_{t}, t)$, achieving a more stable optimization process.

\begin{table*}[t!]
	\centering  
	% \scriptsize
    % \vspace{-8pt}
	\setlength{\tabcolsep}{12pt}
	\begin{tabular}{c|cc|cc|c}  
		\toprule 
    \multirow{2}{*}{\textbf{Method}}
    & \multirow{2}{*}{\textbf{Inference Steps}}
    & \multirow{2}{*}{\textbf{Inference Time (s)}}
    & \multicolumn{2}{c|}{\textbf{Appearance}}
    & \multicolumn{1}{c}{\textbf{Geometry}}\\
    & 
    &
    & $\textbf{FD}_\textbf{incep}\!\downarrow$
    & $\textbf{FD}_\textbf{dinov2}\!\downarrow$
    & $\textbf{ULIP}_\textbf{I}\!\uparrow$\\
    % & $\textbf{Uni3D}_\textbf{I}\!\uparrow$\\
        \midrule

    LGM\textsuperscript{*} & -- & 5 & 26.31 & 322.71& -- \\ 
    3DTopia-XL\textsuperscript{*} & 25 & 5  & 37.68  & 437.37 &	--  \\ 
    Ln3Diff\textsuperscript{*} &  250 & 8 & 26.61  & 357.93  & --  \\ 
    TRELLIS\textsuperscript{*} & 25 $\times$ 2  & -- & 9.35 & 67.21   & -- \\
    \midrule
    TRELLIS & 25 $\times$ 2 & 6.1 & 11.80 & 65.24 & 39.53 \\
    FlashVDM & 5 & 1.30 &  -- &  -- &   37.91\\
     \midrule
    \multirow{2}{*}{Ours} & 1 $\times$ 2 & 0.68  & 18.09  & 164.2 & 36.88  \\
    & 2 $\times$ 2 & 0.94  & 14.16  & 110.9 & 39.11  \\
	\bottomrule
	\end{tabular}  
	% \vspace{-8pt}
        \caption{Quantitative comparison on LGM~\cite{tang2024lgm}, 3DTopia-XL~\cite{3dtopiaxl}, Ln3Diff~\cite{lan2024ln3diff}, and our teacher model TRELLIS~\cite{trellis}. * denotes that the metrics are from TRELLIS, which are measured on the subset of the Toys4K dataset~\cite{toys4k}, and the inference time comes from their original paper. The other reported inference times measured by us are based on evaluations performed on a NVIDIA A800 GPU.}
	\label{tab:comparison}  
\end{table*} 

\begin{figure*}[ht!]
\centering
\vspace{-10pt}
\includegraphics[width=1\textwidth]{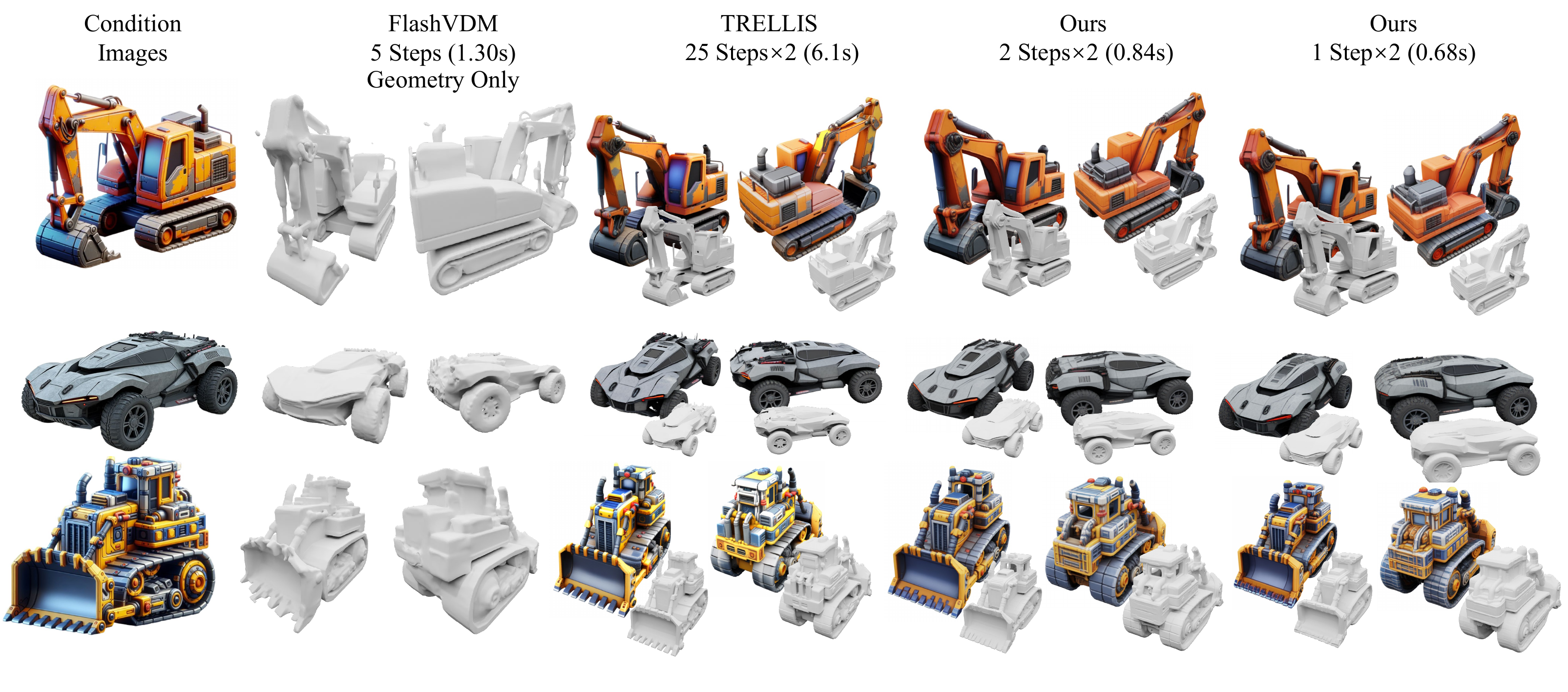}
\vspace{-20pt}
\caption{Qualitative results of FlashVDM~\cite{lai2025flashvdm}, the teacher model TRELLIS~\cite{trellis}, and our method. Since FlashVDM does not generate the appearance of 3D assets, we compare with FlashVDM only on geometry.}
\vspace{-10pt}
% \Description{}
\label{fig:maincompare}
\end{figure*}

\begin{algorithm}[t!]
    \scriptsize
    \caption{\label{alg:vm_vd}\ours~Training Procedure}
    \KwIn{Pretrained flow model $\boldsymbol{v}_{\text{pretrain}}$, dataset $q_{\text{data}}$}
    \KwOut{Trained generator $\boldsymbol{\phi}_\theta$}

    \tcp{Initialize generator from pretrained model}
    $\boldsymbol{\phi}_\theta \leftarrow \text{copyWeights}(v_{\text{pretrain}})$

    \While{not converge}{
        Sample time step $t \sim \mathcal{U}(0,1)$

        \text{~}
        
        \tcp{Compute Velocity Matching Loss}
        
        Sample $\boldsymbol{x}_0 \sim q_{\text{data}}$, $\boldsymbol{z} \sim \mathcal{N}(\boldsymbol{0}, \boldsymbol{I})$
        
        $\boldsymbol{x}_t \leftarrow (1-t)\boldsymbol{x}_0 + t \boldsymbol{z}$
        
        $\boldsymbol{x}_{t-\Delta t} \leftarrow \boldsymbol{x}_t - \boldsymbol{v}_{\text{pretrain}}(\boldsymbol{x}_t, t) \cdot \Delta t$
        
        $\boldsymbol{u}_\theta(\boldsymbol{x}_t, t) \leftarrow \boldsymbol{\phi}_\theta(\boldsymbol{x}_t, t) + t \cdot \frac{1}{\Delta t} \cdot \text{sg} \left( \boldsymbol{\phi}_\theta(\boldsymbol{x}_t, t) - \boldsymbol{\phi}_\theta(\boldsymbol{x}_{t-\Delta t}, t-\Delta t) \right)$
        
        $\mathcal{L}_{\text{VM}} \leftarrow \mathbb{E}_{t, \boldsymbol{x}_0, \boldsymbol{z}} \left[ \left\| \boldsymbol{u}_\theta(\boldsymbol{x}_t, t) - \boldsymbol{v}_{\text{pretrain}}(\boldsymbol{x}_t, t) \right\|^2 \right]$

        \text{~}
        
        \tcp{Compute Velocity Distillation Loss}

        Sample $\boldsymbol{z}', \boldsymbol{z}'' \sim \mathcal{N}(\boldsymbol{0}, \boldsymbol{I})$
        
        $\boldsymbol{x}' \leftarrow \boldsymbol z' - \boldsymbol{\phi}_\theta(\boldsymbol{z}', 1)$
        
        $\boldsymbol{x}'_t \leftarrow (1-t) \boldsymbol{x}' + t \boldsymbol{z}''$

        $\boldsymbol{x}'_{t-\Delta t} \leftarrow \boldsymbol{x}'_t - \boldsymbol{v}_{\text{pretrain}}(\boldsymbol{x}'_t, t) \cdot \Delta t$
        
        $\boldsymbol{u}_\theta(\boldsymbol{x}'_t, t) \leftarrow \boldsymbol{\phi}_\theta(\boldsymbol{x}'_t, t) + t \cdot \frac{1}{\Delta t} \cdot \text{sg} \left( \boldsymbol{\phi}_\theta(\boldsymbol{x}'_t, t) - \boldsymbol{\phi}_\theta(\boldsymbol{x}'_{t-\Delta t}, t-\Delta t) \right)$

        $\mathcal{L}_{\text{VD}} \leftarrow \mathbb{E}_{t, \boldsymbol{z}', \boldsymbol{z}''} \left[ -\left( \boldsymbol{u}_\theta(\boldsymbol{x}'_t, t) - \boldsymbol{v}_{\text{pretrain}}(\boldsymbol{x}'_t, t) \right) \cdot \frac{\partial \boldsymbol{x}'_t}{\partial \theta} \right]$

        \text{~}
        
        \tcp{Update model}
        
        $\mathcal{L} \leftarrow \mathcal{L}_{\text{VM}} + \lambda \mathcal{L}_{\text{VD}}$

        $\boldsymbol{\phi}_\theta \leftarrow \text{update}(\boldsymbol{\phi}_\theta, \mathcal{L})$
    }
\end{algorithm}

\subsection{Velocity Distillation}\label{sec:vd}
Though Eq.~\ref{eq:loss_func1} serves as a tractable objective, we clarify that with only the velocity matching loss the student model $\boldsymbol \phi_{\theta}$ cannot be optimized well. 
An inherent flaw in the velocity matching loss is that both the continuous and discrete formulation of $\boldsymbol \phi_{\theta}(\boldsymbol{x}_t, t)$ cannot be back-propagated properly. Therefore, the optimization of Eq.~\ref{eq:loss_func1} provides a biased gradient estimate.

We revisit our primary objective from the perspective of score distillation~\cite{poole2022dreamfusion, wang2023prolificdreamer}. Since the student and teacher model have the same target distribution $q_{\text{data}}$, we can equivalently convert the constraint on the marginal-data transport into the constraint on the marginal distribution.
Then the primary objective turns to minimize the Kullback-Leibler divergence between the student marginal $p_{\theta}^t$ and the marginal $q_t$ (note we approximate the real marginal with the teacher marginal):
\begin{equation}
    \min_{\theta}\mathbb{E}_t\left[D_{\mathrm{KL}}\left(p_{\theta}^t\parallel q_t\right)\right],
\end{equation}
which is the objective of our velocity distillation.
Specifically, a sample $\boldsymbol{x}' = \boldsymbol{z}' - \boldsymbol{\phi}_\theta(\boldsymbol{z}')$ is first synthesized from a noise sample $\boldsymbol{z}'\sim\mathcal{N}(\boldsymbol{0},\boldsymbol{I})$. Then $\boldsymbol{x}'$ is diffused with $\boldsymbol{z}'' \sim\mathcal{N}(\boldsymbol{0},\boldsymbol{I})$ to obtain $\boldsymbol{x}'_t = (1-t) \boldsymbol{x}' + t \boldsymbol{z}''$. The gradient of the training objective can then be written as
\begin{equation}\begin{aligned}
 & \nabla_\theta\mathbb{E}_t\left[D_{\mathrm{KL}}\left(p_{\theta}^t\parallel q_t\right)\right] \\
 & =\mathbb{E}_{t,\boldsymbol{z}',\boldsymbol{z}''}\left[\nabla_\theta\left(\log p_{\theta}^t(\boldsymbol{x}'_t)-\log q_t(\boldsymbol{x}'_t)\right)\right] \\
 & =\mathbb{E}_{t,\boldsymbol{z}',\boldsymbol{z}''}\left[\left(\nabla_{\boldsymbol{x}'_t}\log p_{\theta}^t(\boldsymbol{x}'_t)-\nabla_{\boldsymbol{x}'_t}\log q_t(\boldsymbol{x}'_t)\right)\frac{\partial\boldsymbol{x}'_t}{\partial\theta}\right].\label{eq:kl_div}
\end{aligned}\end{equation}
ProlificDreamer~\cite{wang2023prolificdreamer} points out that $\nabla_{\boldsymbol{x}'_t}\log p_{\theta}^t(\boldsymbol{x}'_t)$ and $\nabla_{\boldsymbol{x}'_t}\log q_t(\boldsymbol{x}'_t)$ represent the score~\cite{song2020score} of the noisy prediction and the noisy real data respectively.
Here we use the student velocity fields $\boldsymbol u_{\theta}(\boldsymbol{x}'_t, t)$ to replace $-\nabla_{\boldsymbol{x}'_t}\log p_{\theta}^t(\boldsymbol{x}'_t)$, and the teacher velocity fields  $\boldsymbol{v}_{\textrm{pretrain}}(\boldsymbol{x}'_{t}, t)$ to replace $-\nabla_{\boldsymbol{x}'_t}\log q_t(\boldsymbol{x}'_t)$. A detailed proof of the rationale for this choice can be found in Appendix~\ref{sec:proof_vd}. Substituted with Eq.~\ref{eq:u_theta}, the gradient of our velocity distillation loss is formulated as 

\begin{multline}
\nabla_\theta\mathcal{L}_{\mathrm{VD}}(\theta) \coloneqq \\
\mathbb{E}_{t,\boldsymbol{z}',\boldsymbol{z}''}\Big[ -\Big( \boldsymbol \phi_{\theta}(\boldsymbol x'_{t}, t) + t\frac{\mathrm d \boldsymbol \phi_{\theta}(\boldsymbol x'_{t}, t)}{\mathrm d t} - \boldsymbol v_{\textrm{pretrain}}(\boldsymbol x'_{t}, t) \Big)\frac{\partial \boldsymbol x'_t}{\partial\theta}\Big].\label{eq:loss_func2}
\end{multline}
Intuitively, Eq.~\ref{eq:loss_func2} performs probability density distillation which serves as a soft supervision. It applies $\boldsymbol u_{\theta}(\boldsymbol{x}'_t, t) - \boldsymbol{v}_{\textrm{pretrain}}(\boldsymbol{x}'_{t}, t)$ as a criterion to measure the discrepancy between the teacher and student marginals, and indirectly optimizes $\boldsymbol \phi_{\theta}$ through optimizing $\boldsymbol{x}'_{t}$.

\subsection{Optimization}\label{sec:final_loss}
Our velocity matching loss and velocity distillation loss are complementary. While both are d erived from the primary objective, velocity matching provides partially right gradient estimates, improving the performance of $\boldsymbol \phi_{\theta}(\cdot, t)$ for all the time steps $t$. Velocity distillation further utilizes the optimized $\boldsymbol \phi_{\theta}(\cdot, t)$ as a criterion for measuring distribution discrepancy, enhancing the one-step performance of $\boldsymbol \phi_{\theta}$. Our final loss function is formulated as
\begin{equation}
    \mathcal{L}_{\text{\ours}} = \mathcal{L}_{\text{VM}} + \lambda \mathcal{L}_{\text{VD}},
\end{equation}
where $\lambda$ is a hyper-parameter and we set it to be 1.0. Algorithm~\ref{alg:vm_vd} outlines our final training procedure.

\subsection{Relation to Prior Work}\label{sec:relation_prev}
Velocity matching and velocity distillation are both derived from our primary objective.
Velocity matching is related to consistency models~\cite{song2023consistency} and MeanFlow~\cite{geng2025mean}, but still has essential differences. Consistency models are derived from the consistency constraint between adjacent time steps. Compared with consistency models, our objective learns $\boldsymbol \phi_{\theta}(\boldsymbol{x}_t, t)$ more consistent with $\boldsymbol{v}_{\textrm{pretrain}}(\boldsymbol{x}_{t}, t)$, leading to more stable optimization. MeanFlow needs two time variables as model input, thus being not suitable for distilling a pretrained 3D flow model. Differently, our method is motivated by learning the marginal-data transport, and is designed to distill a pretrained flow model.

Velocity distillation is related to Score Distillation Sampling (SDS)~\cite{poole2022dreamfusion} and Variational Score Distillation (VSD)~\cite{wang2023prolificdreamer}. SDS and VSD both try to measure the student and teacher marginals with the score, i.e., the gradient of the log probability density. SDS only uses the added noise as the student score, leading to a single-point Dirac distribution estimation~\cite{wang2023prolificdreamer}. VSD finetunes an additional diffusion model $\epsilon_{\phi}$ to learn the student distribution, resulting in additional memory cost and further learning errors. Instead, our method only learns one model to be both the generator and the distribution measure, being both low-cost and accurate.

\begin{figure*}[h]
\centering
\includegraphics[width=1\textwidth]{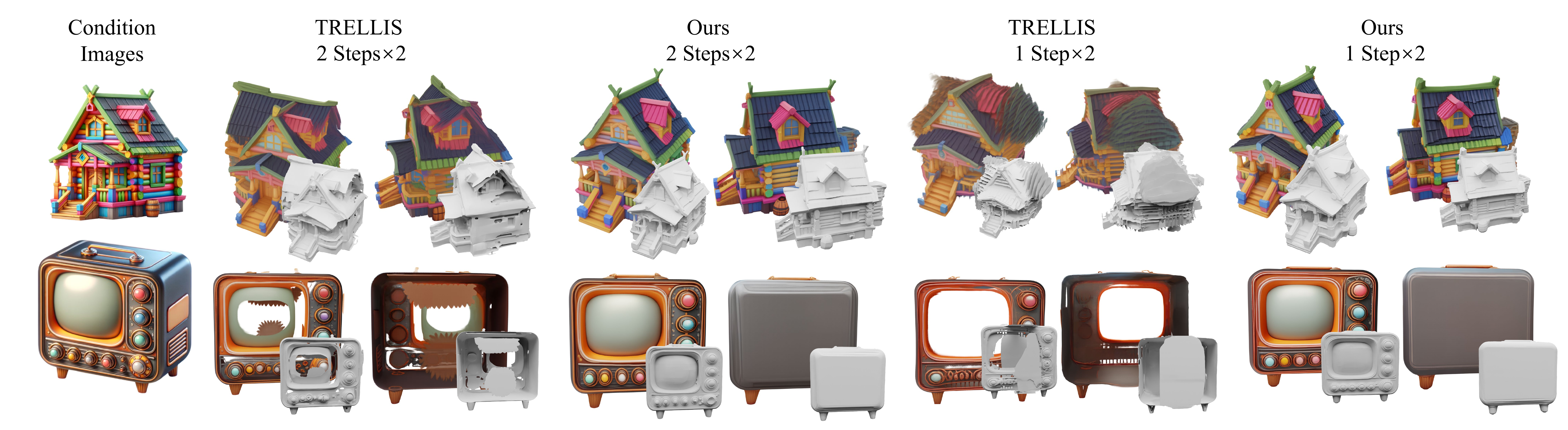}
\vspace{-15pt}
\caption{Qualitative results with and without distillation during few-step inference.}
% \Description{}
\label{fig:fewsteps_abla}
\end{figure*}

\begin{figure*}[h]
\centering
\vspace{-7pt}
\includegraphics[width=1\textwidth]{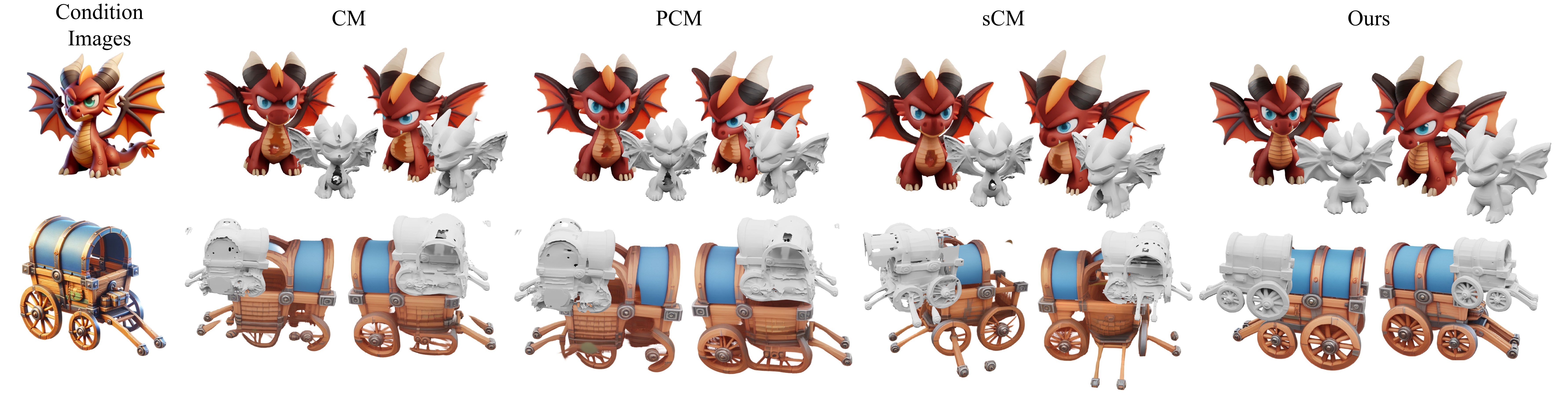}
\vspace{-20pt}
\caption{Qualitative comparison on CM~\cite{song2023consistency}, PCM~\cite{wang2024phased}, sCM~\cite{lu2024simplifying} and ours. Our method exhibits the most complete and fine-grained geometric and visual details.}
% \Description{}
\label{fig:cm_abla}
\end{figure*}

\begin{figure}[t]
\centering
\includegraphics[width=1\columnwidth]{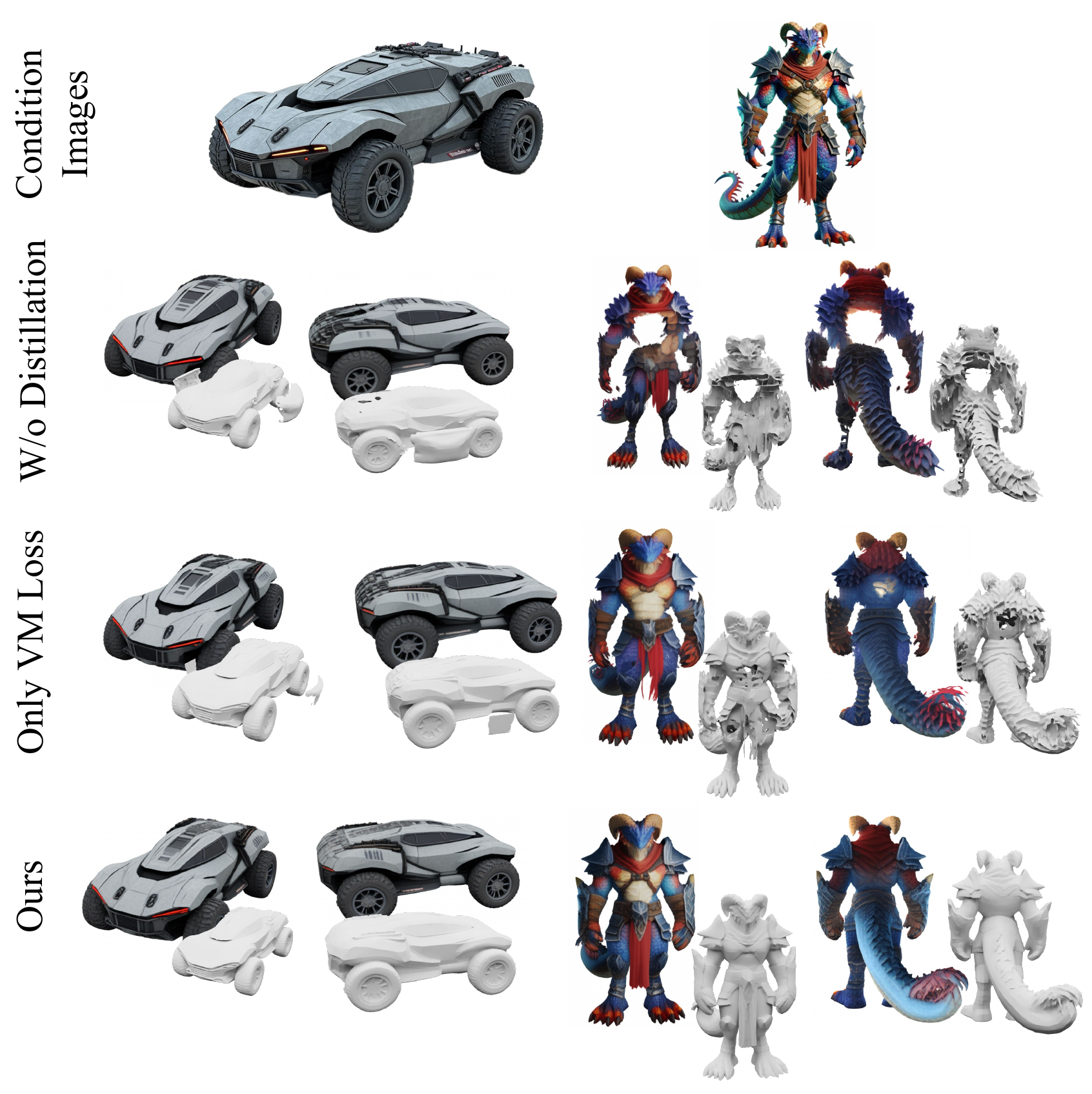}
\vspace{-15pt}
\caption{Qualitative results of ablation studies on our proposed two loss functions.}
% \Description{}
\label{fig:vm_vd}
\end{figure}

\section{Experiments}
\paragraph{Implementation Details.}
Our training details follow that of TRELLIS. We only train the generation stage of TRELLIS, the flow transformer for structure generation, and the sparse flow transformer for structured latents generation. For each iteration, we randomly select one image from 24 conditional images as the input condition. The conditional images are rendered with different FoVs. We use AdamW~\cite{loshchilov2017decoupled} as our optimizer. The classifier-free guidance (CFG)~\cite{ho2022classifier} in training is set to be 40 for VM and 100 for VD. We train the structure flow transformer for 6k steps with learning rate $4e-8$ and the sparse flow transformer for 4k steps with learning rate $1e-8$, both with batch size 256.
We finetune the pretrained 1.1B (Large) image model of TRELLIS. The sampler during inference is the same as TRELLIS.

\paragraph{Datasets.}
For training, we use around 500K 3D assets from the Objaverse (XL)~\cite{deitke2023objaverse}, ABO~\cite{collins2022abo}, 3D-FUTURE~\cite{3dfuture}, and HSSD datasets~\cite{khanna2023hssd}.
For evaluation, we use all the 3,218 3D assets from the Toys4k dataset~\cite{toys4k}.

\paragraph{Metrics.}
For quantitative evaluation, we use the Fréchet Distance (FD) with different feature extractors. For appearance, we use the Inceptionv3 model~\cite{xia2017inception} ($\text{FD}_\text{incep}$) and the DINOv2 model~\cite{oquab2023dinov2} ($\text{FD}_\text{dinov2}$) as feature extractors. Following TRELLIS~\cite{trellis}, we render the generated 3D assets at an elevation of 30 degrees and at four azimuth angles (0, 90, 180, and 270 degrees), resulting in rendering four views per asset. The ground truth 3D assets are rendered in the same way, and we compute the FD between the generated and ground truth images. 
% For geometry, we use the PointNet++ model as the feature extractor to compute ${ULIP}_{I}$. We sample 4,000 points from both the generated mesh and the ground truth mesh, and then calculate the FD between them.
For geometry, we follow FlashVDM~\cite{lai2025flashvdm} and Direct2D-S2~\cite{wu2025direct3ds2} to use ULIP-2~\cite{xue2024ulip2} for calculating the similarity between the images and the final mesh. To accurately measure the similarity between the images and the mesh, we use the four rendered views mentioned above, calculate the similarity for each view, and then take the average. For the metric $\text{ULIP}_\text{I}$ in the paper, we multiply the results by 100.

\begin{table}[t!]
\newcommand{\xmark}{\ding{55}}
\centering
% \vspace{0pt}
\setlength{\tabcolsep}{6pt}
\begin{tabular}{c|ccc}
\toprule
  & $\textbf{FD}_\textbf{incep}\!\downarrow$ & $\textbf{FD}_\textbf{dinov2}\!\downarrow$ & $\textbf{ULIP}_\textbf{I}\!\uparrow$ \\
\midrule
 CM & 20.06  & 194.5 & 34.62 \\
   PCM & 19.53  & 189.5 & 34.96 \\
  sCM & 19.29 & 186.0 & 35.04\\
  Ours & 18.09  & 164.2 & 36.88 \\
 \bottomrule
\end{tabular}
\caption{Quantitative comparison on CM~\cite{song2023consistency}, PCM~\cite{wang2024phased}, sCM~\cite{lu2024simplifying} and ours.}
\label{tab: cm}
\vspace{-5pt}
\end{table}

\begin{table}[t!]
\newcommand{\xmark}{\ding{55}}
\centering
% \vspace{0pt}
\setlength{\tabcolsep}{6pt}
\begin{tabular}{cc|ccc}
\toprule
$\mathcal{L}_{\text{VM}}$ & $\mathcal{L}_{\text{VD}}$ & $\textbf{FD}_\textbf{incep}\!\downarrow$ & $\textbf{FD}_\textbf{dinov2}\!\downarrow$ & $\textbf{ULIP}_\textbf{I}\!\uparrow$ \\
\midrule
 $\times$ & $\times$  & 20.26 & 195.6 & 34.62 \\
 $\checkmark$ &  $\times$  & 18.42 & 172.0& 35.99  \\
 $\checkmark$ &  $\checkmark$  & 18.09  & 164.2 & 36.88 \\
 \bottomrule
\end{tabular}
\caption{Ablation studies on our proposed two loss functions.}
\label{tab: loss func}
\vspace{-10pt}
\end{table}

\subsection{Quantitative Comparison}
Since methods for distilling 3D diffusion models remain scarce, FlashVDM~\cite{lai2025flashvdm} is one of the few. We also compare against non-distilled methods (LGM~\cite{tang2024lgm}, 3DTopia-XL~\cite{3dtopiaxl}, Ln3Diff~\cite{lan2024ln3diff}) and the teacher model TRELLIS~\cite{trellis}. Apart from LGM, all others are based on 3D diffusion. The results are shown in Table~\ref{tab:comparison}. 
% Results for LGM, 3DTopia-XL, and Ln3Diff are taken from the TRELLIS paper (evaluated on 1,250 samples from the Toys4k dataset~\cite{toys4k}). TRELLIS, FlashVDM, and our method are evaluated on all 3,218 samples from the filtered Toys4k dataset released by TRELLIS.

\paragraph{Non-diffusion-based Method.}
We compare our method with LGM~\cite{tang2024lgm}, one of the most influential non-3D-diffusion-based 3D generation methods. 
% LGM uses a multiview 2D diffusion model~\cite{shi2023mvdream,wang2023imagedream} to generate images from multiple viewpoints, then generates the 3D Gaussians to each view and concatenates them. 
Compared to LGM, our approach distills an end-to-end 3D diffusion model, which has stronger 3D awareness, improves both appearance and geometry quality of the generated 3D assets, and reduces inference time.

\paragraph{Non-distilled 3D Diffusion Methods.}
We further compare with non-distilled 3D diffusion methods, including 3DTopia-XL~\cite{3dtopiaxl}, Ln3Diff~\cite{lan2024ln3diff}, and the teacher model TRELLIS~\cite{trellis}. After distillation, our method reduces the two 3D flow transformers of TRELLIS from 25 steps each ($50$ steps total) to $1$ and $1$ steps ($2$ steps in total), achieving high-quality 3D generation in just $0.68\text{s}$. 
% We achieve comparable performance on appearance metrics ($FD_{incep}$, $FD_{dinov2}$) and the geometry metric ($ULIP_{I}$).
We achieve comparable performance on appearance metrics ($\text{FD}_\text{incep}$, $\text{FD}_\text{dinov2}$) and the geometry metric ($\text{ULIP}_\text{I}$) compared with the teacher model when the number of inference steps is 4.
% Compared to other 3D diffusion methods such as Ln3Diff and 3DTopia-XL~\cite{3dtopiaxl}, our method cuts the number of denoising steps from $250$ to $2$ (a $125\times$ reduction) while achieving better performace. 
Compared to Ln3Diff and 3DTopia-XL~\cite{3dtopiaxl}, our method outperforms it in both appearance and geometry evaluations with fewer than $12\times$ the inference steps and inference time.

\paragraph{Distilled 3D Diffusion Method.}
Compared to the recent distilled 3D-shape diffusion model FlashVDM~\cite{lai2025flashvdm}, we evaluate only its geometry metric ($\text{ULIP}_\text{I}$), since it does not generate appearance. We use the officially released code of FlashVDM and evaluate it with the default configuration. Our method and FlashVDM complete inference in a similar number of steps. On the geometry metric $\text{ULIP}_\text{I}$, our method outperforms FlashVDM.

\subsection{Qualitative Comparison}
% For better comparison, we present a visual comparison of the results. Fig.~\ref{fig:maincompare} shows the visual results of FlashVDM, TRELLIS, and our method. Furthermore, Fig.~\ref{fig:fewsteps_abla} presents the visual results before and after distillation using our method.

\paragraph{Qualitative Comparison with Other Methods.}
We present a visual comparison of the results in Fig.~\ref{fig:maincompare}. The results show that our method can generate high-quality 3D assets with both geometry and appearance. Since FlashVDM~\cite{lai2025flashvdm} does not generate appearance for 3D assets, our comparison with FlashVDM mainly focuses on geometry. Compared to FlashVDM, our method generates geometry with more details and better consistency with the conditioned images. 
% For example, in the first row of Fig.~\ref{fig:maincompare}, the hydraulic rod in front of the excavator bucket appears twice in the result generated by FlashVDM, while there is only one in the original image. In contrast, our method generates the hydraulic rod in front of the excavator bucket consistent with the original image. In the second row of Fig.~\ref{fig:maincompare}, the wheel details generated by our method are also more realistic. In the third row of Fig.~\ref{fig:maincompare}, the serrated details of the front shovel are more complete in our results than in those of FlashVDM. 
Compared with the teacher model TRELLIS~\cite{trellis}, our method generates comparable 3D assets.

% \paragraph{Qualitative Comparison Before and After Distillation.}
% In Fig.~\ref{fig:fewsteps_abla}, we present the results of the model before and after distillation under few-step inference. The model before distillation generates 3D assets with broken and hollow structures when using few inference steps. For example, in Fig.~\ref{fig:fewsteps_abla}, the television generated by the model before distillation has incomplete geometry, while the model after distillation produces more complete geometry under few-step inference. In addition, for the cottage in Fig.~\ref{fig:fewsteps_abla}, our method generates appearance that is more consistent with the condition image after distillation. The model after distillation generates 3D assets with higher quality in both geometry and appearance under few-step inference.

\subsection{Comparison with CM methods}
We conduct qualitative and quantitative comparison with three CM methods: CM~\cite{song2023consistency},  PCM~\cite{wang2024phased}, and sCM~\cite{lu2024simplifying}. The results are shown in Fig.~\ref{fig:cm_abla} and Table~\ref{tab: cm}. As the basic consistency model, CM performs the worst. Due to the multi-phase design, PCM learns more stably and achieves a better performance than CM. sCM utilizes a continuous format of CM, which eliminates the discretization errors, thereby performing better than the other two CM methods. Derived from a different objective, our method largely outperforms all the CM methods, with more complete and fine-grained geometric and visual details.

\subsection{Ablation Study}
We conduct ablation studies on our proposed two loss functions to validate their effectiveness. 

\paragraph{Velocity Matching.} 
The quantitative results of the ablation study on VM are in Table~\ref{tab: loss func}. It can be found that directly applying VM to finetune TRELLIS improves both the geometric quality and the visual quality. In Fig.~\ref{fig:vm_vd}, applying VM significantly reduces the unwanted extra geometry generated by the model, while partially restoring the missing and incomplete geometric structures. Moreover, compared to the generation result of the non-distilled model, the distilled model exhibits enhanced geometric fidelity, particularly in fine structures such as the shape of car windows and the overall body geometry.

\paragraph{Velocity Distillation.} 
The quantitative results of the ablation study on VD are in Table~\ref{tab: loss func}. With the help of VM, VD is capable of learning the geometry better. In Fig.~\ref{fig:vm_vd}, VD further eliminates the unwanted extra geometry and addresses the remaining incomplete regions in the geometry.

\paragraph{Joint Optimization.}
In Table~\ref{tab: loss func}, joint optimization with VM and VD achieves the best performance, demonstrating the complementarity between the two loss functions. In Fig.~\ref{fig:fewsteps_abla}, under the joint effect of the two loss functions, both geometry and visual quality are significantly improved compared to the TRELLIS baseline. Our method significantly addresses the remaining incomplete regions in the geometry and improves the details.

\section{Conclusion}
We present a novel framework MDT-dist for few-step flow distillation in 3D generation. By formulating a primary objective as modeling the transport from the marginal distribution to the data distribution, our approach provides a more direct solution to few-step generation, in contrast to consistency models that rely on consistency constraints on the adjacent time steps. 
To effectively optimize this objective, we introduce two optimization objectives: Velocity Matching (VM), which converts the optimization target from the transport level to the velocity level, enabling tractable and stable matching of velocity fields between the student and the teacher, and Velocity Distillation (VD), which converts the optimization target from the transport level to the distribution level, leveraging the learned velocity fields to perform probability density distillation.
When applied to TRELLIS, our method reduces the sampling steps from 25 to 1–2 while preserving high geometric and visual fidelity. Extensive experiments demonstrate that our approach significantly outperforms existing consistency distillation techniques, and achieves new state-of-the-art performance in few-step 3D generation.
\paragraph{Limitations and Future Work.}
Our method still requires a large amount of conditional images and geometric data to conduct few-step distillation training. While images are relatively easy to collect, high-quality geometric data is much more scarce and expensive, making the few-step 3D generation distillation costly. A possible improvement is to eliminate the dependence on geometric data and take only conditional images as input, which will reduce the cost and further scale up the distillation by leveraging the abundant online images.

{
    \small
    \bibliographystyle{ieeenat_fullname}
    \bibliography{main}
}
\appendix
\clearpage
\maketitlesupplementary

\section{Proofs}

\subsection{Velocity Matching}\label{sec:proof_vm}
\begin{theorem}
Assume the velocity matching loss is uniformly bounded above by a constant \(M > 0\), i.e., for all \(t \in [0,1]\), \(\boldsymbol{x}_0 \sim q_{\text{data}}\), and \(\boldsymbol{z} \sim \mathcal{N}(\boldsymbol 0, \boldsymbol I)\),
\begin{equation*}
    \|\boldsymbol{u}_{\theta}(\boldsymbol{x}_t, t) - \boldsymbol{v}_{\textrm{pretrain}}(\boldsymbol{x}_t, t)\|^2 \leq M.
\end{equation*}
Under this assumption, the primary objective with $D(\cdot, \cdot)$ replaced by MSE metric satisfies the error bound
\begin{equation*}
L_{\text{primary}}(\theta) \leq M \cdot \mathbb{E}_t [t^2],
\end{equation*}
where the expectation is over the distribution of \(t\). For a fixed \(t\), the bound simplifies to \(L_{\text{primary}}(\theta \mid t) \leq M t^2\).
\end{theorem}

\begin{proof}
Define the learnable student transport function to be $\boldsymbol T_{\theta}(t) = t \boldsymbol{\phi}_{\theta}(\boldsymbol{x}_t, t)$, and the target integral $\boldsymbol T_{\text{pretrain}}(t) = \int_0^t \boldsymbol{v}_{\textrm{pretrain}}(\boldsymbol{x}_{\tau}, \tau) \, \mathrm{d}\tau$. Note that \(\boldsymbol T_{\theta}(0) = \boldsymbol T_{\text{pretrain}}(0) = \boldsymbol{0}\), and their derivatives satisfy \(\frac{\mathrm{d} \boldsymbol T_{\theta}(t)}{\mathrm{d} t} = \boldsymbol{u}_{\theta}(\boldsymbol{x}_t, t)\) and \(\frac{\mathrm{d} \boldsymbol T_{\text{pretrain}}(t)}{\mathrm{d} t} = \boldsymbol{v}_{\textrm{pretrain}}(\boldsymbol{x}_t, t)\).
The error for a fixed \(t\), \(\boldsymbol{x}_0\), and \(\boldsymbol{z}\) is
\begin{equation}
\begin{aligned}
D(\boldsymbol T_{\theta}(t), \boldsymbol T_{\text{pretrain}}(t)) &= \left\| \int_0^t \left( \boldsymbol{u}_{\theta}(\boldsymbol{x}_{\tau}, \tau) - \boldsymbol{v}_{\textrm{pretrain}}(\boldsymbol{x}_{\tau}, \tau) \right) \mathrm{d}\tau \right\|^2 \\
&= \left\| \int_0^t \boldsymbol{e}(\tau) \, \mathrm{d}\tau \right\|^2,
\end{aligned}
\end{equation}
where \(\boldsymbol{e}(\tau) = \boldsymbol{u}_{\theta}(\boldsymbol{x}_{\tau}, \tau) - \boldsymbol{v}_{\textrm{pretrain}}(\boldsymbol{x}_{\tau}, \tau)\) satisfies \(\|\boldsymbol{e}(\tau)\|^2 \leq M\) for all \(\tau \in [0,t]\) by assumption (with \(D\) as squared norm).
By the triangle inequality for norms,
\begin{equation}
\left\| \int_0^t \boldsymbol{e}(\tau) \, \mathrm{d}\tau \right\| \leq \int_0^t \|\boldsymbol{e}(\tau)\| \, \mathrm{d}\tau \leq \int_0^t \sqrt{M} \, \mathrm{d}\tau = t \sqrt{M},
\end{equation}
since \(\|\boldsymbol{e}(\tau)\| \leq \sqrt{M}\). Squaring both sides gives
\begin{equation}
\left\| \int_0^t \boldsymbol{e}(\tau) \, \mathrm{d}\tau \right\|^2 \leq M t^2.
\end{equation}
Thus, for fixed \(t\),
\begin{equation}
L_{\text{primary}}(\theta \vert t) = \mathbb{E}_{\boldsymbol{x}_0, \boldsymbol{z}} \left[ D(\boldsymbol T_{\theta}(t), \boldsymbol T_{\text{pretrain}}(t)) \right] \leq M t^2,
\end{equation}
as the bound holds uniformly. Taking expectation over \(t\),
\begin{equation}\begin{aligned}
L_{\text{primary}}(\theta) &= \mathbb{E}_t \mathbb{E}_{\boldsymbol{x}_0, \boldsymbol{z}} \left[ D(\boldsymbol T_{\theta}(t), \boldsymbol T_{\text{pretrain}}(t)) \right] \\ &\leq \mathbb{E}_t [M t^2] = M \mathbb{E}_t [t^2].
\end{aligned}\end{equation}
\end{proof}

This proof suggests that the error in the primary objective is bounded by the error in the velocity matching loss. Therefore, we can learn the primary objective through optimizing the velocity matching loss.

\subsection{Velocity Distillation}\label{sec:proof_vd}

Score-SDE~\cite{song2020score} defines the forward diffusion process by the forward-time SDE:
\begin{equation}
\mathrm{d}\boldsymbol{x}_t = \boldsymbol{f}(\boldsymbol{x}_t, t) \, \mathrm{d}t + g(t) \, \mathrm{d}\boldsymbol{w}_t,
\end{equation}
where \(\boldsymbol{f}: \mathbb{R}^d \times [0,T] \to \mathbb{R}^d\) is a Lipschitz-continuous drift function, \(g: [0,T] \to \mathbb{R}_{>0}\) is a continuous diffusion coefficient, and \(\boldsymbol{w}_t\) is a standard \(d\)-dimensional Wiener process.

The score function at time \(t\) is the gradient of the log-marginal density:
\begin{equation}
\boldsymbol{s}(\boldsymbol{x}, t) = \nabla_{\boldsymbol{x}} \log p_t(\boldsymbol{x}),
\end{equation}
where \(p_t(\boldsymbol{x})\) is the marginal density induced by the forward process at time \(t\).

The velocity field arises in the deterministic reformulation of the generative process via the probability flow ODE:
\begin{equation}
\frac{\mathrm{d}\boldsymbol{x}_t}{\mathrm{d}t} = \boldsymbol{v}(\boldsymbol{x}_t, t) = \boldsymbol{f}(\boldsymbol{x}_t, t) - \frac{1}{2} g(t)^2 \boldsymbol{s}(\boldsymbol{x}_t, t). \label{eq:velocity_score}
\end{equation}

\begin{theorem}
Velocity distillation differs from score distillation by a multiplicative factor $\frac{1}{2} g(t)^2$.
\end{theorem}

\begin{proof}
We have
\begin{equation}\begin{aligned}
\nabla_\theta\mathcal{L}_{\mathrm{VD}}&(\theta \vert t) \\
&=\mathbb{E}_{\boldsymbol{z},\boldsymbol{z}'}\Big[ -\Big(\boldsymbol u_{\theta}(\boldsymbol x'_{t}, t) - \boldsymbol v_{\textrm{pretrain}}(\boldsymbol x'_{t}, t)  \Big)\frac{\partial \boldsymbol x'_t}{\partial\theta} \Big] \\
&=\mathbb{E}_{\boldsymbol{z},\boldsymbol{z}'}\Big[ -\Big(\boldsymbol f(\boldsymbol x'_{t}, t) - \frac{1}{2} g(t)^2 \boldsymbol{s}_{\theta}(\boldsymbol{x}_t, t) \\
&~~~~- \big(\boldsymbol f(\boldsymbol x'_{t}, t) - \frac{1}{2} g(t)^2 \boldsymbol{s}'_{\text{pretrain}}(\boldsymbol{x}'_t, t) \big)  \Big)\frac{\partial \boldsymbol x'_t}{\partial\theta} \Big]\\
&=\mathbb{E}_{\boldsymbol{z},\boldsymbol{z}'}\Big[ \Big( \frac{1}{2} g(t)^2 (\boldsymbol{s}_{\theta}(\boldsymbol{x}'_t, t) - \boldsymbol{s}_{\text{pretrain}}(\boldsymbol{x}'_t, t)) \Big)\frac{\partial \boldsymbol x'_t}{\partial\theta} \Big]\\
&= \frac{1}{2} g(t)^2 \nabla_\theta\mathcal{L}_{\mathrm{SD}}(\theta \vert t).
\end{aligned}\end{equation}
\end{proof}

This proof implies that, omitting a weight coefficient related to $t$, velocity distillation is equivalent to score distillation.
\section{More Results}
% We present more qualitative results on our methods in Fig.~\ref{fig:step2} and Fig.~\ref{fig:step1}. Our method achieves high-fidelity geometric and visual quality, with only $2 \times 2$ and $1 \times 2$ steps at inference respectively.

\begin{figure*}[h]
\centering
\includegraphics{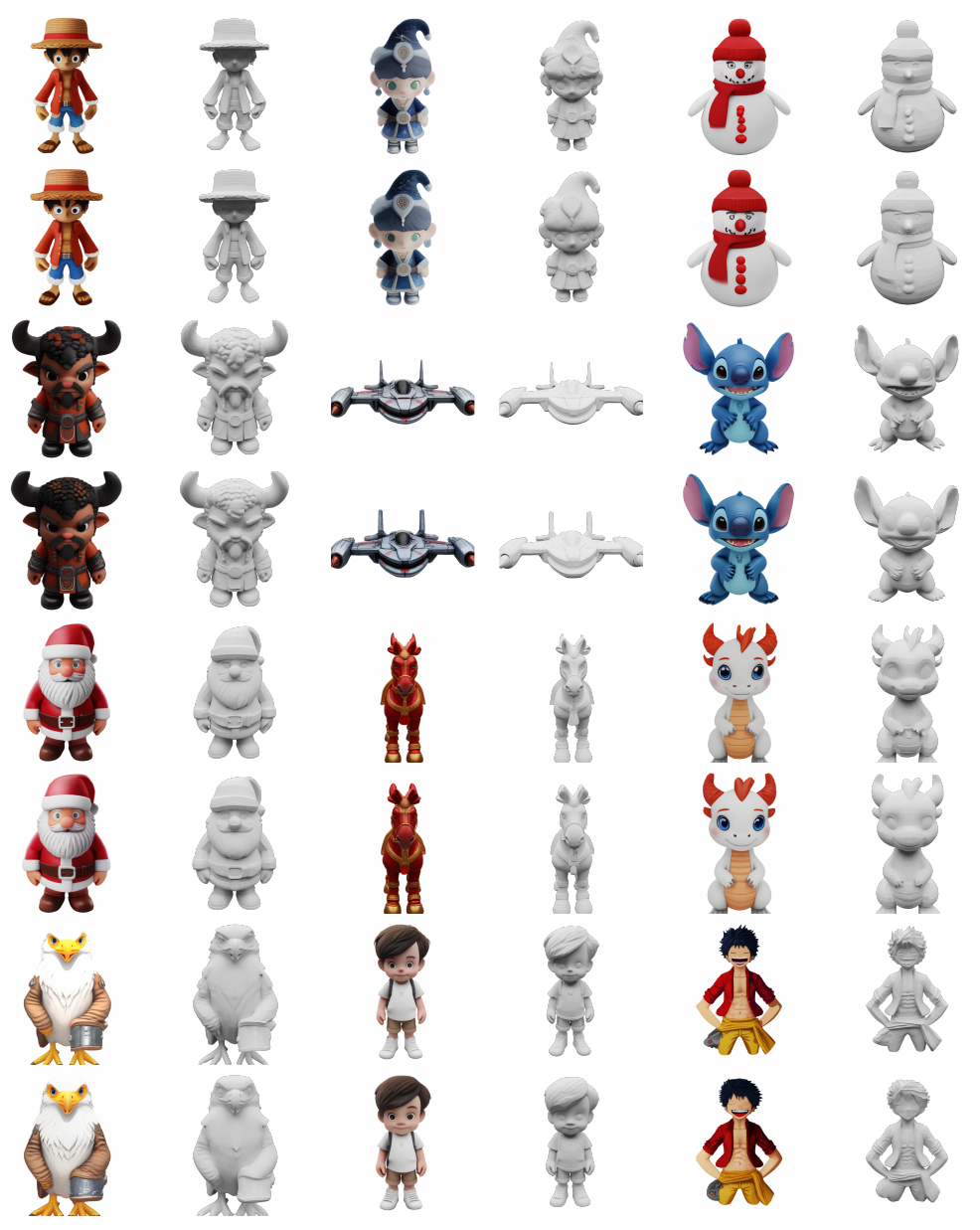}
\caption{Our 3D Gaussians~\cite{kerbl3Dgaussians} and mesh generation results. Odd rows and even rows represent samples from $2$ steps $\times 2$ and $1$ step $\times 2$ during inference, respectively.}
\label{fig:step2}
\end{figure*}

\begin{figure*}[h]
\centering
\includegraphics{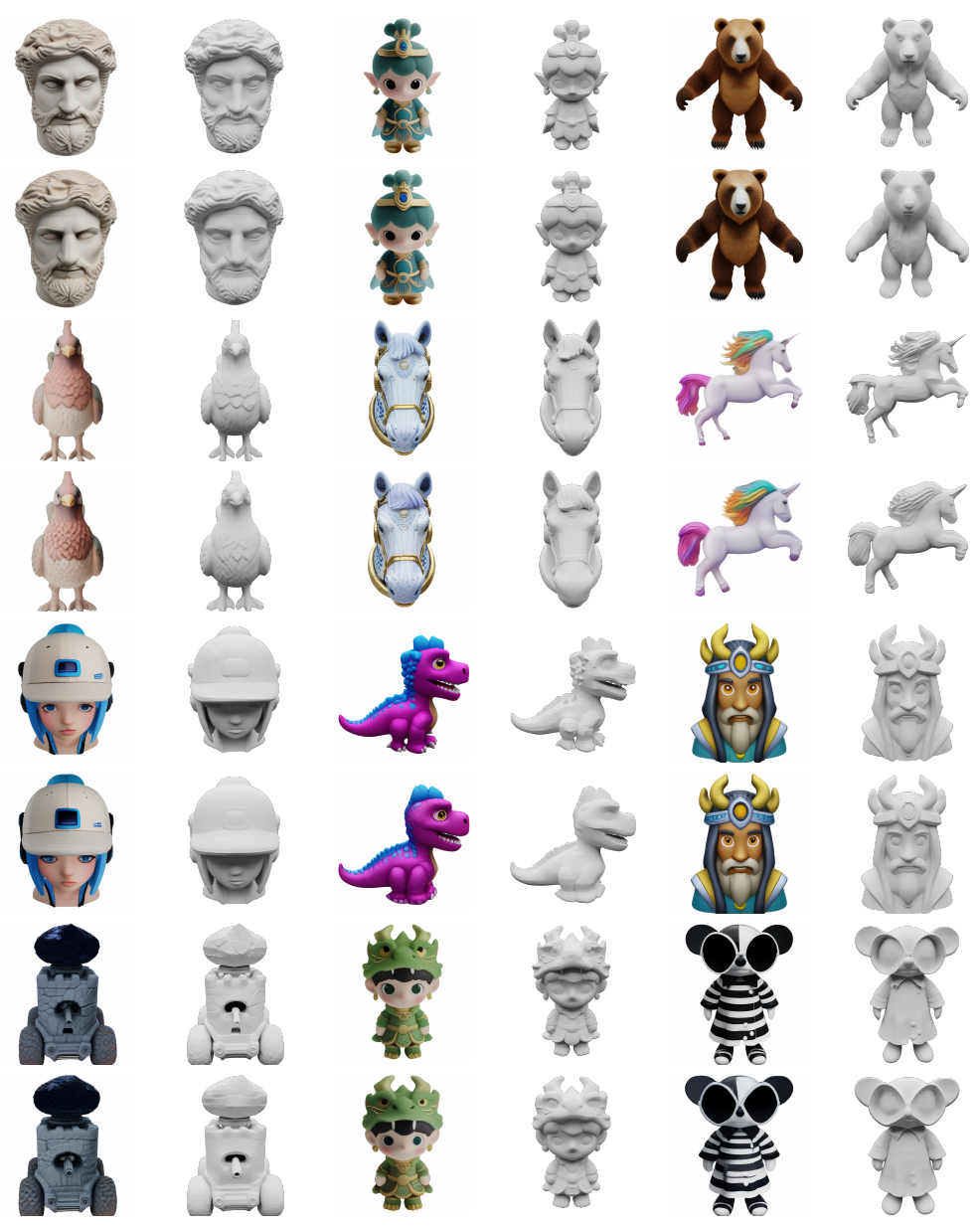}
\caption{Our 3D Gaussians~\cite{kerbl3Dgaussians} and mesh generation results. Odd rows and even rows represent samples from $2$ steps $\times 2$ and $1$ step $\times 2$ during inference, respectively.}
\label{fig:step1}
\end{figure*}

\end{document}